\documentclass[12pt]{alt2024} 

\usepackage[utf8]{inputenc} 
\usepackage[T1]{fontenc}    
\usepackage{hyperref}       
\usepackage{url}            
\usepackage{booktabs}       
\usepackage{amsfonts}       
\usepackage{nicefrac}       
\usepackage{microtype}      

\usepackage{multirow}
\usepackage{hhline}
\usepackage{color, colortbl}
\definecolor{Gray}{gray}{0.9}
\usepackage{rotating}
\usepackage{tablefootnote}
\usepackage{algorithm}
\usepackage{algorithmic}
\usepackage{xcolor}

\newtheorem{assumption}[theorem]{Assumption}

\newcommand{\ee}{\mathbb{E}}
\newcommand{\prob}{\mathbb{P}}
\newcommand{\norm}[1]{\left\lVert#1\right\rVert}

\newcommand{\qed}{\hfill \ensuremath{\blacksquare}}
\newcommand{\fqed}{\tag*{$\blacksquare$}}
\usepackage{mathtools}
\DeclarePairedDelimiter\ceil{\lceil}{\rceil}

\newif\ifincludeappendix
\includeappendixtrue 

\ifincludeappendix
\usepackage{xr}
\else
\usepackage{xr}
\fi

\title[Provably Efficient Exploration in Constrained RL: Posterior Sampling Is All You Need]{Provably Efficient Exploration in Constrained Reinforcement Learning: Posterior Sampling Is All You Need}
\usepackage{times}

\altauthor{%
 \Name{Danil Provodin}\textsuperscript{1,2} \Email{d.provodin@tue.nl}\\
 \Name{Pratik Gajane}\textsuperscript{1} \Email{p.gajane@tue.nl}\\
 \Name{Mykola Pechenizkiy}\textsuperscript{1} \Email{m.pechenizkiy@tue.nl}\\
 \Name{Maurits Kaptein}\textsuperscript{2,3} \Email{m.c.kaptein@tilburguniversity.edu}\\
 \addr \textit{Eindhoven University of Technology}, Eindhoven, the Netherlands \\
 \addr \textit{Jheronimus Academy of Data Science}, ‘s-Hertogenbosch, the Netherlands\\
 \addr \textit{Tilburg University}, Tilburg, the Netherlands
}

\begin{document}

\maketitle

\begin{abstract}%
  We present a new algorithm based on posterior sampling for learning in Constrained Markov Decision Processes (CMDP) in the infinite-horizon undiscounted setting. 
  The algorithm achieves near-optimal regret bounds while being advantageous empirically compared to the existing algorithms.
  Our main theoretical result is a Bayesian regret bound for each cost component of $\tilde{O} (HS \sqrt{AT})$ for any communicating CMDP with $S$ states, $A$ actions, and bound on the hitting time $H$. This regret bound matches the lower bound in order of time horizon $T$ and is the best-known regret bound for communicating CMDPs in the infinite-horizon undiscounted setting. Empirical results show that, despite its simplicity, our posterior sampling algorithm outperforms the existing algorithms for constrained reinforcement learning.%
\end{abstract}


\section{Introduction}

Reinforcement learning (RL) refers to the problem of learning by trial and error in sequential decision-making systems based on the scalar signal aiming to minimize the total cost accumulated over time.
In many situations, however, the desired properties of the agent behavior are better described using constraints, as a single objective might not suffice to explain the real-life setting. For example, a robot should not only fulfill its task but should also control its wear and tear by limiting the torque exerted on its motors \citep{tessler2018reward}; recommender platforms should not only focus on revenue growth but also optimize users long-term engagement \citep{asfar2021_RLforRS:survey}; and autonomous driving vehicles should reach the destination in a time and fuel-efficient manner while obeying traffic rules \cite{Le2019BatchPL}. A natural approach for handling such cases is specifying the problem using multiple objectives, where one objective is optimized subject to constraints on the others.

A typical way of formulating the constrained RL problem is a Constrained Markov Decision Process (constrained MDP or CMDP) \citep{Altman99constrainedmarkov}, which proceeds in discrete time steps. At each time step, the system occupies a \textit{state}, and the decision maker chooses an \textit{action} from the set of allowable actions. As a result of choosing the action, the decision maker receives a (possibly stochastic) vector of \textit{costs}, and the system then transitions to the next state according to a fixed \textit{state transition distribution}. In the reinforcement learning problem, the underlying state transition distributions and/or cost distributions are unknown and need to be learned from observations while aiming to minimize the total cost. This requires the algorithm to balance between exploration and exploitation, i.e., exploring different states and actions to learn the system dynamics vs.\ exploiting available information to minimize costs.

One way to balance exploration and exploitation, which is widely studied in reinforcement learning literature, is the \textit{optimism in the face of uncertainty} (OFU) principle \cite{LAI19854}. The OFU principle involves maintaining tight overestimates of expected outcomes and selecting actions with the highest optimistic estimate. This principle encourages exploration since poorly-learned states and actions will have higher estimates due to greater uncertainty. 

Another popular algorithm design principle is \textit{posterior sampling} \cite{THOMPSON_1933}. Posterior sampling maintains a posterior distribution for the unknown parameters from which it samples a plausible model. Then, posterior sampling solves for and executes the policy that is optimal under the sampled model.
Unlike the OFU principle, posterior sampling guides the exploration by the variance of the posterior distribution. Both these principles underpin many algorithms in reinforcement learning \cite{regal_2009, JMLR:v11:jaksch10a, Osband_PSRL2013, AY_bayesian_control_2015, NIPS2017_Shipra_OPSRL, TS_MDP_Ouyang_2017, lattimore_szepesvári_2020}. 

Posterior sampling has several advantages over OFU algorithms. First, unlike many OFU methods that simultaneously optimize across a set of plausible models (see, e.g., \cite{regal_2009, JMLR:v11:jaksch10a} for unconstrained RL and \cite{Efroni_2020_CMDP, Singh_CMDP_2020, chen_2022_optimisQlearn} for constrained RL), posterior sampling only requires solving for an optimal policy for a single sampled model, making it computationally more efficient in situations where sampling from a posterior is inexpensive (see \cite{Osband_PSRL2013, TS_MDP_Ouyang_2017, jafarniajahromi2021online} for unconstrained RL).  
Second, while OFU algorithms require explicit construction of confidence bounds based on observed data, which is a complicated statistical problem even for simple models, in posterior sampling, uncertainty is quantified in a statistically efficient way through the posterior distribution \cite{whyPSbetter}. This makes it easier to set up and implement, which is highly desirable.

However, posterior sampling faces a challenge in constrained problems and has been under-explored. Specifically, one key challenge of posterior sampling in constrained RL is to guarantee the feasibility of the problem with respect to the sampled model, i.e., to ensure the existence of a policy that satisfies the constraints with respect to the sampled model. A recent work \cite{kalagarla2023safe} makes a restricted assumption that every sampled CMDP is feasible and establishes Bayesian regret bounds in the episodic setting.

In this work, we study the posterior sampling principle in constrained reinforcement learning and address this challenge by providing novel results on the feasibility of the sampled CMDP. We focus on the most general \textit{infinite-horizon undiscounted average cost} criterion \cite{Altman99constrainedmarkov}. Under this setting, the learner-environment interaction never ends or resets, and the goal of achieving optimal long-term average performance under constraints, appears to be much more challenging compared to the episodic setting. In our setting, the underlying CMDP is assumed to be communicating with (unknown) finite bound on the hitting time $H$ and have finite states $S$ and finite actions $A$. By utilizing a mild assumption of the existence of a strictly feasible (unknown) solution to the original CMDP, we guarantee that the sampled CMDP becomes feasible after $O ( \sqrt{T} )$ steps. This allows for maintaining the advantages of posterior sampling in constrained problems. 

\textbf{Our main contribution} is a posterior sampling-based algorithm (\textsc{PSConRL}). Under a Bayesian framework, we show that the expected regret of the algorithm accumulated up to time $T$ is bounded by $\tilde{O} ( HS \sqrt{AT} )$ for each of the cost components, where $\tilde{O}$ hides logarithmic factors.   
Drawing inspiration from the algorithmic design structure of \cite{TS_MDP_Ouyang_2017}, the algorithm proceeds in episodes with two stopping criteria. At the beginning of every episode, it samples transition probability vectors from a posterior distribution for every state and action pair and executes the optimal policy of the sampled CMDP. To solve this planning problem, we utilize a linear program (LP) in the space of occupancy measures that incorporates constraints directly \cite{Altman99constrainedmarkov}. The optimal policy computed for the sampled CMDP is used throughout the episode. 

Thus, the main result of the paper shows that near-optimal Bayesian regret bounds are achievable in constrained RL under the infinite-horizon undiscounted setting. Our proofs combine the proof techniques of \cite{TS_MDP_Ouyang_2017} with that of \cite{NIPS2017_Shipra_OPSRL}. Additionally, simulation results demonstrate that our algorithm significantly outperforms existing approaches for three CMDP benchmarks.

The rest of the paper is organized as follows. 
Section \ref{sec:problem_formulation} is devoted to the methodological setup and contains the problem formulation. 
The \textsc{PSConRL} algorithm is introduced in Section \ref{sec:algorithm}.
Analysis of the algorithm is presented in Section \ref{sec:regret_bound}, which is followed by numerical experiments in Section \ref{sec:experiments}. 
Section \ref{sec:lit_review} briefly reviews the previous related work. 
Finally, we conclude with Section \ref{sec:conclusion}.

\section{Problem formulation}
\label{sec:problem_formulation}
\subsection{Constrained Markov Decision Processes}

A constrained MDP model is defined as a tuple $M = (\mathcal{S}, \mathcal{A}, p, \textit{\textbf{c}}, \tau)$ where $\mathcal{S}$ is the state space, $\mathcal{A}$ is the action space, $p : \mathcal{S} \times \mathcal{A} \xrightarrow{} \Delta^{\mathcal{S}}$ is the transition function, with $\Delta^{\mathcal{S}}$ indicating simplex over $\mathcal{S}$, $\textit{\textbf{c}} : \mathcal{S} \times \mathcal{A} \xrightarrow{} [0, 1]^{m+1}$ is the cost vector function, and $\tau \in [0,1]^m$ is a cost threshold. In general, CMDP is an MDP with multiple cost functions ($c_0, c_1,\dots,c_m$), one of which, $c_0$, is used to set the optimization objective, while the others, ($c_1,\dots,c_m$), are used to restrict what policies can do. A stationary policy $\pi$ is a mapping from state space $\mathcal{S}$ to a probability distribution on the action space $\mathcal{A}$, $\pi : \mathcal{S} \xrightarrow{} \Delta^{\mathcal{A}}$, which does not change over time.  Let $S = |\mathcal{S}|$ and $A = |\mathcal{A}|$ where $|\cdot|$ denotes the cardinality.

For transitions $p$ and scalar cost function $c$, a stationary policy $\pi$ induces a Markov chain and the expected infinite-horizon average cost (or \textit{loss}) for state $s \in \mathcal{S}$ is defined as
\begin{equation}
    J^{\pi}(s; c, p) = \overline{\lim}_{T \to \infty} \frac{1}{T} \sum_{t=1}^T \mathbb{E}_{p}^{\pi} \left [ c(s_t,a_t) \lvert s_0 = s \right ]
\end{equation}
where $\mathbb{E}_{p}^{\pi}$ is the expectation under the probability measure $\mathbb{P}_{p}^{\pi}$ over the set of infinitely long state-action trajectories. $\mathbb{P}_{p}^{\pi}$ is induced by policy $\pi$, transition function $p$, and the initial state $s$. Given some fixed initial state $s$  and  $\tau_1, \dots , \tau_m \in \mathbb{R}$ , the CMDP optimization problem is to find a policy  $\pi$  that minimizes $J^\pi(s; c_0,p)$ subject to the constraints  $J^\pi(s; c_i,p) \leq \tau_i, i = 1, \dots, m$:
\begin{equation}
    \min_\pi J^\pi(s; c_0, p) \quad \text{ s.t. } \quad  J^\pi(s; c_i, p) \leq \tau_i,\, i=1,\dots,m \,. 
    \label{eq:objective_cost}
\end{equation}

Next, we introduce hitting and cover times of Markov chains induced by stationary policies. Let $\pi$ be an arbitrary policy and $P_{\pi}$ be a transition matrix of a Markov chain $(S_0, S_1, \dots )$ induced by policy $\pi$. For $s, s' \in \mathcal{S}$, define the hitting time to be the first time at which the chain visits state $s'$ starting from $s$, $\tau_{ss'} = \min\{t \geq 0 : S_0=s, S_t=s'\}$, and $\Pi_{finite}$ to be a set of policies for which $\max_{s, s' \in \mathcal{S}}\mathbb{E}^{\pi}_{p}\tau_{ss'}$ is finite. Then, the worst-case hitting time between states in a chain is defined as follows 
\begin{equation*}
    t_{hit}(p) = \max_{s, s' \in \mathcal{S}} \max_{\pi \in \Pi_{finite}}\mathbb{E}^{\pi}_{p}\tau_{ss'}.
\end{equation*}
The cover time $\tau_{cov}$ is the minimal value such that, for every state $s' \in \mathcal{S}$, there exists $t \leq \tau_{cov}$ with $S_t = s'$. In other words, the cover time is the expected length of a random walk to cover every state at least once. 

To control the regret vector (defined below), we consider the subclass of communicating CMDPs. CMDP is communicating if for every pair of states $s$ and $s^{\prime}$ there exists a stationary policy (which might depend on $s, s^{\prime}$) under which $s^{\prime}$ is accessible from $s$. We note that in communicating CMDP neither hitting times nor cover times are guaranteed to be finite for a given policy $\pi$, however, by \cite[Proposition 8.3.1]{Puterman_mdp}, there exists a policy which induces a Markov chain with finite hitting and cover time. Therefore, $\Pi_{finite}$ is not empty for communicating CMDPs, and $t_{hit}$ is well defined.

We define $\Omega_*$ to be the set of all transitions $p$ such that the CMDP with transition probabilities $p$ is communicating, and there exists a number $H$ such that $t_{hit}(p) \leq H$. We will focus on CMDPs with transition probabilities in set $\Omega_*$.

Next, by \cite[Theorem 8.2.6]{Puterman_mdp}, for scalar cost function $c$, transitions $p$ that corresponds to communicating CMDP, and stationary policy $\pi$, there exists a bias function $v(s;c,p)$ satisfying the \textit{Bellman equation} for all $s \in \mathcal{S}$:
\begin{equation}
\label{eq:bellman}
    J^{\pi}(s; c, p) + v^{\pi}(s; c, p) = c(s,a) + \sum_{s' \in \mathcal{S}} p(s'|s,a) v^{\pi}(s'; c, p).
\end{equation}

If $v$ satisfies the Bellman equation, $v$ plus any constant also satisfies the Bellman equation. Furthermore, the loss of the optimal stationary policy $\pi_{\ast}$ does not depend on the initial state, i.e., $J^{\pi_{\ast}}(s; c, p) = J^{\pi_{\ast}}(c, p)$, as presented in \cite[Theorem 8.3.2]{Puterman_mdp}. Without loss of generality, let $\min_{s \in \mathcal{S}} v^{\pi_*}(s; c_i, p) = 0$, for $i=1, \dots m$, and define the span of the MDP as $sp(p) = \max_{1 \leq i \leq m}\max_{s \in \mathcal{S}}v^{\pi_*}(s; c_i, p)$. Note, if $t_{hit}(p) \leq H$, then $sp(p) \leq H$ as well \citep{regal_2009}.

Given the above definitions and results, we can now define the reinforcement learning problem studied in this paper.

\subsection{The reinforcement learning problem}

We study the reinforcement learning problem where an agent interacts with a communicating CMDP $M = (\mathcal{S}, \mathcal{A}, p_{\ast}, \textit{\textbf{c}}, \tau)$. We assume that the agent has complete knowledge of $\mathcal{S}, \mathcal{A}$, and the cost function $\textit{\textbf{c}}$, but not the transitions $p_{\ast}$ or the hitting time bound $H$. This assumption is common for RL literature \citep{regal_2009, NIPS2017_Shipra_OPSRL, whyPSbetter, kalagarla2023safe} and is without loss of generality because the complexity of learning the cost and reward functions is dominated by the complexity of learning the transition probability.

We focus on a Bayesian framework for the unknown parameter $p_{\ast}$. That is, at the beginning of the interaction, the actual transition probabilities $p_{\ast}$ are randomly generated from the prior distribution $f_1$. The agent can use past observations to learn the underlying CMDP model and decide future actions. The goal is to minimize the total cost $\sum_{t=1}^T c_0(s_t, a_t)$ while violating constraints as little as possible, or equivalently, minimize the total regret for the main cost component and auxiliary cost components over a time horizon $T$, defined as
\begin{align*}
    BR_+ (T;c_0) & = \mathbb{E} \left [  \sum_{t=1}^T \big (c_0(s_t, a_t) - J^{\pi_{\ast}}(c_0; p_{\ast}) \big )_+\right ], \\
    BR_+ (T;c_i) & = \mathbb{E} \left [  \sum_{t=1}^T \big (  c_i(s_t, a_t) - \tau_i \big )_+ \right ], \quad \quad i=1,\dots,m,
\end{align*}
where $s_t, a_t, t = 1, \dots , T$ are generated by the agent, $J^{\pi_{\ast}}(c_0; p_{\ast})$ is the optimal loss of the CMDP $M$, and $[x]_+ := \max \{0,x\}$. The above expectation is with respect to the prior distribution $f_1$, the randomness in the state transitions, and the randomized policy.

\subsection{Assumptions}
We introduce two mild assumptions that are common in reinforcement learning literature.

\begin{assumption}
The support of the prior distribution $f_1$ is a subset of $\Omega_*$. That is, the CMDP $M$ is communicating and $t_{hit}(p_{\ast}) \leq H$.
\label{assum:WASP}
\end{assumption}

This type of assumption is common for the Bayesian framework (see, e.g., \cite{TS_MDP_Ouyang_2017, Agarwal_2021_PSRL}) and is not overly restrictive \citep{regal_2009, chen_2022_optimisQlearn}. We also provide a practical justification for this assumption in the experiments section by showing that it can be supported by choosing Dirichlet distribution as a prior.

\begin{assumption}
\label{assum:slater}
    There exists $\gamma > 0$ and unknown policy $\bar{\pi}(\cdot | s) \in \Delta^{\mathcal{A}}$ such that $J^{\bar{\pi}}(c_i, p_{\ast}) \leq \tau_i - \gamma$ for all $i \in \{1, \dots, m \}$, and without loss of generality, we assume under such policy $\bar{\pi}$, the Markov chain resulting from the CMDP is irreducible and aperiodic.
\end{assumption}

The first part of the assumption is standard in constrained reinforcement learning (see, e.g., \cite{Efroni_2020_CMDP, DingWYWJ21}) and is mild as we do not require the knowledge of such policy. The second part is without loss of generality due to \cite[Proposition 8.3.1]{Puterman_mdp} and \cite[Proposition 8.5.8]{Puterman_mdp}. By imposing this assumption, we can control the sensitivity of problem in Eq. \eqref{eq:objective_cost} to the deviation between the true and sampled transitions. Later, we will use this assumption to guarantee that the constrained problem in Eq. \eqref{eq:objective_cost} becomes feasible under the sampled transitions once the number of visitations to every
state-action pair is sufficient.

\section{Posterior sampling algorithm}
\label{sec:algorithm}
In this section, we propose the Posterior Sampling for Constrained Reinforcement Learning (\textsc{PSConRL}) algorithm. The algorithm proceeds in episodes and uses \textit{Linear Programming} for solving CMDP and \textit{Bayes rule} for the posterior update.

\paragraph{Linear programming in CMDP.} When CMDP is known, an optimal policy for \eqref{eq:objective_cost} can be obtained by solving the following linear program \citep{Altman99constrainedmarkov}:

\begin{align}
    \min_{\mu} \sum_{s,a} \mu(s,a) c_0(s,a), \label{eq1}\\
    \mathrm{s.t.}\quad \sum_{s,a} \mu(s,a) c_i(s,a) \leq \tau_i,\, \quad i=1,\dots,m, \\
    \sum_a \mu(s,a) = \sum_{s', a} \mu(s', a) p(s',a,s), \quad \forall s \in \mathcal{S}, \\
    \mu(s,a) \geq 0, \quad \forall (s,a) \in \mathcal{S} \times \mathcal{A}, \quad \sum_{s,a} \mu(s,a) = 1, \label{eq4}
\end{align}
where the decision variable $\mu(s,a)$ is occupancy measure (fraction of visits to $(s,a)$). Given the optimal solution for LP \eqref{eq1}-\eqref{eq4}, $\mu_*(s,a)$, one can construct the optimal stationary policy $\pi_{\ast}(a|s)$ for \eqref{eq:objective_cost} by choosing action $a$ in state $s$ with probability $\frac{\mu_*(s,a)}{\sum_{a'} \mu_*(s,a')}$.

\paragraph{Bayes rule.} At each time step $t$, given the history $h_t$, the agent can compute the posterior distribution $f_t$ given by $f_t(\mathcal{P}) = \prob(p_{\ast} \in \mathcal{P} \lvert h_t)$ for any set $\mathcal{P}$. Upon applying the action $a_t$ and observing the new state $s_{t+1}$, the posterior distribution at $t + 1$ can be updated according to Bayes’ rule as

\begin{equation}
\label{eq:bayes_rule}
    f_{t+1}(dp) = \frac{p(s_{t+1} | s_t, a_t) f_t(dp)}{\int p^{\prime}(s_{t+1}|s_t, a_t) f_t(dp^{\prime})}.
\end{equation}

\paragraph{Algorithm description.} At the beginning of episode $k$, a parameter $p_k$ is sampled from the posterior distribution $f_{t_k}$, where $t_k$ is the start of the $k$-th episode. During each episode $k$, actions are generated from the optimal stationary policy $\pi_k$ for the sampled parameter $p_k$, which is observed by solving LP \eqref{eq1}-\eqref{eq4}. The key challenge of the posterior sampling algorithm is that neither problem in Eq. \eqref{eq:objective_cost} nor LP \eqref{eq1}-\eqref{eq4} are guaranteed to be feasible under the sampled transitions $p_k(\cdot|s,a)$. To account for this issue, the algorithm performs an additional check to verify whether the LP (\ref{eq1})-(\ref{eq4}) is feasible under $p_k(\cdot|s,a)$ (line 7), otherwise it recovers the uniformly random policy $\pi_0$.\footnote{Note that computationally this step is polynomial in the number of constraints (i.e., $poly(SA)$).} 
Using Assumption \ref{assum:slater}, we will show that eventually, after $O ( \sqrt{T} )$ steps, the sampled CMDP becomes feasible, and the algorithm will effectively compute $\pi_k$ by solving LP \eqref{eq1}-\eqref{eq4}.

Let $N_t(s,a)$ denote the number of visits to $(s,a)$ before time $t$ and $T_k = t_{k+1} - t_k$ be the length of the episode. We use two stopping criteria of \cite{TS_MDP_Ouyang_2017} for episode construction. The rounds $t= 1,...,T$ are broken into consecutive episodes as follows: the $k$-th episode begins at the round $t_k$ immediately after the end of $(k - 1)$-th episode and ends at the first round $t$ such that (i) $N_{t}(s,a) \geq 2 N_{t_k}(s,a)$ or (ii) $t \leq t_k + T_{k-1}$ for some state-action pair $(s,a)$. The first criterion is the doubling trick of \cite{JMLR:v11:jaksch10a} and ensures the algorithm has visited some state-action pair $(s,a)$ at least the same number of times it had visited this pair $(s,a)$ before episode $k$ started. The second criterion controls the growth rate of episode length and is believed to be necessary under the Bayesian setting \citep{TS_MDP_Ouyang_2017}.

The algorithm is summarized in Algorithm \ref{alg1:psrl_transitions}.

\begin{algorithm}[tb]
   \caption{Posterior Sampling for Constrained Reinforcement Learning (\textsc{PSConRL})}
   \label{alg1:psrl_transitions}
\begin{algorithmic}[1]
    \STATE {\bfseries Input:} $f_1$

    \STATE Initialization: $t \gets 1$, $t_k \gets 0$, $\pi_0(\cdot) \gets \frac{1}{\lvert \mathcal{A} \rvert}$
    \FOR{episodes $k = 1, 2, \dots$}
        \STATE $T_{k-1} \gets t - t_k$ 
        \STATE $t_k \gets t$
        \STATE Generate $p_k(\cdot|s,a) \sim f_{t_k}$
        \IF{ LP (\ref{eq1})-(\ref{eq4}) is feasible under $p_k(\cdot|s,a)$}
            \STATE Compute $\pi_k(\cdot)$ by solving LP (\ref{eq1})-(\ref{eq4})
        \ELSE
            \STATE $\pi_k(\cdot) \gets \pi_{0}(\cdot)$
        \ENDIF
        \REPEAT
            \STATE Apply action $a_t = \pi_k(s_t)$
            \STATE Observe new state $s_{t+1}$
            \STATE Update counter $N_t(s_t,a_t)$
            \STATE Update $f_{t+1}$ according to \eqref{eq:bayes_rule}
            \STATE $t \gets t + 1$
        \UNTIL{ $t \leq t_k + T_{k-1}$ and $N_t(s,a) \leq 2 N_{t_k}(s,a)$} for some $(s,a) \in \mathcal{S} \times \mathcal{A}$
    \ENDFOR

\end{algorithmic}
\end{algorithm}

\paragraph{Main result.} We now provide our main results for the \textsc{PSConRL} algorithm for learning in CMDPs.

\begin{theorem}
\label{thm:regret_bound}
     For any communicating CMDP $M$ with $S$ states and $A$ actions under Assumptions \ref{assum:WASP} and \ref{assum:slater}, the Bayesian regret for main and auxiliary cost components of Algorithm \ref{alg1:psrl_transitions} are bounded as:
     
    \begin{align*}
        & BR_+ (T;c_i) \leq O \left ( H S \sqrt{AT \log (AT)} + H \sqrt{T} \log (S) \right ), \quad \text{for } i = 0, \dots, m. \\
    \end{align*}
    Here $O(\cdot)$ notation hides only the absolute constant.
\end{theorem}

\begin{remark}
    The \textsc{PSConRL} algorithm only requires the knowledge of $\mathcal{S}$, $\mathcal{A}$, $\textbf{\textit{c}}$, and the prior distribution $f_1$. It does not require the knowledge of the horizon $T$ or hitting time bound $H$ as in \cite{chen_2022_optimisQlearn}.
\end{remark}


\section{Theoretical analysis}
\label{sec:regret_bound}

In this section, we prove that the regret of Algorithm \ref{alg1:psrl_transitions} is bounded by $\tilde{O} \left ( HS \sqrt{AT} \right )$. 

A key property of posterior sampling is that conditioned on the information at time $t$, the transition functions $p_{\ast}$ and $p_t$ have the same distribution if $p_t$ is sampled from the posterior distribution at time $t$ \citep{Osband_PSRL2013}. Since the \textsc{PSConRL} algorithm samples $p_k$ at the stopping time $t_k$, we use the stopping time version of the posterior sampling property stated as follows.

\begin{lemma}[Adapted from Lemma 1 of \cite{jafarniajahromi2021online}]
Let $t_k$ be a stopping time with respect to the filtration $\left ( \mathcal{F}_t \right )_{t=1}^{\infty}$, and $p_k$ be the sample drawn from the posterior distribution at time $t_k$. Then, for any measurable function $g$ and any $\mathcal{F}_{t_k}$-measurable random variable $X$, we have 
\label{lm:posterior_lemma}
\begin{align*}
    \ee \left [ g(p_k, X) \right ] = \ee \left [ g(p_{\ast}, X) \right ].
\end{align*}
\end{lemma}

Further, we introduce two lemmas that constitute the primary novel elements in the proof of Theorem \ref{thm:regret_bound}. Recall that in every episode $k$, Algorithm \ref{alg1:psrl_transitions} runs either an optimal loss policy by solving LP \eqref{eq1}-\eqref{eq4} for the sampled transitions or utilizes random policy $\pi_0$. In Lemma \ref{lm:feasibility}, we prove that problem in Eq. \eqref{eq:objective_cost} becomes feasible under the sampled transitions once the number of visitations to every state-action pair is sufficient, i.e., there exists a policy that satisfies constraints in \eqref{eq:objective_cost} and Algorithm \ref{alg1:psrl_transitions} will effectively find an optimal solution.

\begin{lemma}[Feasibility lemma]
\label{lm:feasibility}
    If $N_{t_k}(s,a) \geq \sqrt{T}$, $\norm{p_k(\cdot|s,a) - p_{\ast}(\cdot|s,a)}_1 \leq \sqrt{\frac{14S \log (2A T t_k)}{\max \{1, N_{t_k}(s,a)\}}}$ for all $(s,a)$, and $\gamma \geq H \sqrt{\frac{14S \log (2A T^2)}{\sqrt{T}}}$ there exists policy $\pi$, which satisfies $J^\pi(c_i; p_k) \leq \tau_i$ for all $i \in \{1, \dots, m\}$. 
\end{lemma}

\begin{remark}
    A similar statement was introduced in \cite{Agarwal_2021_PSRL}. However, we emphasize that their analysis was conducted for ergodic CMDPs and does not generalize to our setting.
\end{remark}

\noindent \textbf{Proof Sketch}
    Fix some $i \in \{1, \dots, m\}$. By Assumption \ref{assum:slater}, policy $\bar{\pi}$ is strictly feasible under the true transitions $p_{\ast}$. This is a good candidate to provide a feasible solution with respect to sampled transitions $p_k$ that lie close enough to the true transitions $p_{\ast}$. 
    To prove this, we use the fact that $\norm{p_k(\cdot|s,a) - p_{\ast}(\cdot|s,a)}_1 \leq \sqrt{\frac{14S \log (2A T t_k)}{\max \{1, N_{t_k}(s,a)\}}}$ and the relation between the loss, the bias vector and the cost vector of policy $\bar{\pi}$, discussed in Section \ref{sec:problem_formulation}, and show that the loss of $\bar{\pi}$ w.r.t. $p_k$, $J^{\bar{\pi}}(c_i, p_k)$, lies within distance $\gamma$ from the loss of $\bar{\pi}$ w.r.t. $p_{\ast}$, $J^{\bar{\pi}}(c_i, p_{\ast})$:
    \begin{equation}
        J^{\bar{\pi}}(c_i, p_k) - J^{\bar{\pi}}(c_i, p_{\ast}) \leq \gamma.
        \label{eq:feasible_gain}
    \end{equation}

    Specifically, using the properties of the stationary distribution of an aperiodic and irreducible Markov chain \citep{Puterman_mdp},  we show that
    \begin{equation*}
        J^{\bar{\pi}}(c_i, p_k) - J^{\bar{\pi}}(c_i, p_{\ast}) \leq H \norm{p_k(\cdot|s,a) - p_{\ast}(\cdot|s,a)}_1 \leq \gamma 
    \end{equation*}

    \noindent Rearranging the terms in \eqref{eq:feasible_gain} and using Assumption \ref{assum:slater} gives us the desired result:
    \begin{equation*}
        J^{\bar{\pi}}(c_i, p_k) \leq J^{\bar{\pi}}(c_i, p_{\ast}) + \gamma \leq \tau_i - \gamma + \gamma \leq \tau_i. \fqed
    \end{equation*}

Subsequently, in Lemma \ref{lm:cover_time_up_bound}, we show that the exploration policy $\pi_0$ induces a Markov chain with finite hitting and cover times. This result will play a crucial role in determining an upper bound for when the conditions of Lemma \ref{lm:feasibility} will be met.
\begin{lemma}
\label{lm:cover_time_up_bound}
    For any communicating CMDP $M$, the hitting and cover times of the random uniform policy $\pi_0$ are finite, and $\pi_0 \in \Pi_{finite}$.
\end{lemma}
The proofs of Lemmas \ref{lm:feasibility} and \ref{lm:cover_time_up_bound} are presented in appendices \ref{apx:proof_feasibility} and \ref{apx:upper_bound_cover_time} correspondingly.

\subsection{Proof of Theorem \ref{thm:regret_bound}}

    \paragraph{Bounding regret of the main cost component.}
    To analyze the performance of \textsc{PSConRL} over $T$ time steps, define $K_T = \arg \max \{ k : t_k \leq T \}$, number of episodes of \textsc{PSConRL} until time $T$. By \cite[Lemma 1]{TS_MDP_Ouyang_2017}, $K_T$ is upper-bounded by $\sqrt{2SAT \log(T)}$. Next, decompose the total regret into the sum of episodic regrets conditioned on the good event that the sampled CMDP is feasible. Define the set $\mathcal{G} = \left \{ p \in \Omega_*: \exists \pi \text{ s.t. } J^\pi(c_i; p) \leq \tau_i, \forall i \in \{1, \dots, m\} \right \}$. Then, using the tower rule, we decompose the regret as 
    \begin{align}
    \label{eq:regret_decompose}
        BR_+ (T;c_0) & 
        = \ee \left [ \sum_{t=1}^T \left ( c_0(s_t, a_t) - J^{\pi_*}(c_0; p_{\ast}) \right )_+ \right ] 
        = \sum_{k=1}^{K_T} \ee \left [ R_{0,k} \right ] 
        \notag
        \\
        & 
        = \sum_{k=1}^{K_T} \ee \left [ R_{0,k} \lvert p_k \notin \mathcal{G} \right ] \prob \left ( p_k \notin \mathcal{G} \right )
        + \sum_{k=1}^{K_T} \ee \left [ R_{0,k} \lvert p_k \in \mathcal{G} \right ] \prob \left ( p_k \in \mathcal{G} \right ) 
        \\
        \notag
    \end{align}
    where $R_{0,k} = \sum_{t=t_k}^{t_{k+1}-1} \left [ c_0(s_t, a_t) - (t_{k+1} - t_k)  J^{\pi_{\ast}}(c_0; p_{\ast}) \right ]_+$ and $J^{\pi_*}(c_0; p_{\ast})$ is the optimal loss of CMDP $M$.

    
    Define two events $A_1 = \{p_k \notin \mathcal{G} \wedge  N_{t_k}(s,a) \geq \sqrt{T}, \forall s, a \}$ and $A_2 = \{p_k \notin \mathcal{G} \wedge \exists (s,a):  N_{t_k}(s,a) < \sqrt{T}\}$. Then, the first term of \eqref{eq:regret_decompose} can be further decomposed as
    \begin{align*}
        \sum_{k=1}^{K_T} \ee \left [ R_{0,k} \lvert p_k \notin \mathcal{G} \right ] \prob \left ( p_k \notin \mathcal{G} \right ) 
        & = \sum_{k=1}^{K_T} \ee \left [ R_{0,k} \lvert A_1 \right ] \prob \left ( A_1 \right ) 
        + \sum_{k=1}^{K_T} \ee \left [ R_{0,k} \lvert A_2 \right ] \prob \left ( A_2 \right ). \\
    \end{align*}
    First, we bound $\sum_{k=1}^{K_T} \ee \left [ R_{0,k} \lvert A_1 \right ] \prob \left ( A_1 \right )$. Let $\bar{p}_k(s^{\prime}|s,a) = \frac{N_{t_k}(s,a,s^{\prime})}{N_{t_k}(s,a)}$ be the empirical mean for the transition probability at the beginning of episode $k$, where $N_{t_k}(s,a,s^{\prime})$ is the number of visits to $(s,a,s^{\prime})$. Define the confidence set 
    \begin{align*}
        B_k = \left \{ p : \norm{\bar{p}_k(\cdot|s,a) - p(\cdot|s,a)}_1 \leq \beta_k \right \},
    \end{align*}
    where $\beta_k = \sqrt{\frac{14S \log (2A T t_k)}{\max \{1, N_{t_k}(s,a)\}}}$. 
    
    Now, we observe that $ \left \{ A_1 \right \} \subseteq \left \{ \norm{p_k(\cdot|s,a) - p_{\ast}(\cdot|s,a)}_1 > \beta_k \right \}$, otherwise, by Lemma \ref{lm:feasibility}, problem \eqref{eq:objective_cost} would be feasible under $p_k$, and therefore $p_k \in \mathcal{G}$ which contradicts to $p_k \notin \mathcal{G}$. Next, we note that $B_k$ is $\mathcal{F}_{t_k}$-measurable which allows us to use Lemma \ref{lm:posterior_lemma}. Setting $\delta= 1/T$ in Lemma \ref{lm:jaksch_ci} implies that $\prob \left ( \norm{p_k(\cdot|s,a) - p_{\ast}(\cdot|s,a)}_1 > \beta_k \right )$ can be bounded by $\frac{2}{15Tt_k^6}$. Indeed, 
    \begin{align*}
        \prob \left ( \norm{p_k(\cdot|s,a) - p_{\ast}(\cdot|s,a)}_1 > \beta_k \right ) \leq \prob \left ( p_{\ast} \notin B_k \right ) + \prob \left ( p_k \notin B_k \right ) = 2 \prob \left ( p_{\ast} \notin B_k \right ) \leq \frac{2}{15Tt_k^6},
    \end{align*}
    where the last equality follows from Lemma \ref{lm:posterior_lemma} and the last inequality is due to Lemma \ref{lm:jaksch_ci}. Finally, we obtain
    \begin{align*}
        \sum_{k=1}^{K_T} \ee \left [ R_{0,k} \lvert A_1 \right ] \prob \left ( A_1 \right ) \leq \sum_{k=1}^{K_T} \frac{2 (t_{k+1} - t_k)}{15Tt_k^{6}} \leq \frac{2}{15}\sum_{k=1}^{\infty} k^{-6} \leq 1.
    \end{align*}

    To bound $\sum_{k=1}^{K_T} \ee \left [ R_{0,k} \lvert A_2 \right ] \prob \left ( A_2 \right )$, we recall that Algorithm \ref{alg1:psrl_transitions} executes uniform random policy $\pi_0$, when \eqref{eq:objective_cost} is infeasible. Define $t^{\pi_0}_{hit} = \max_{s, s' \in \mathcal{S}} \mathbb{E}^{\pi_0}_{p}\tau_{ss'}$ to be the hitting time of policy $\pi_0$, and $t^{\pi_0}_{cov} = \max_{s \in \mathcal{S}} \mathbb{E}^{\pi_0}_{p}\tau_{cov}$ to be the cover time of policy $\pi_0$. By Lemma \ref{lm:cover_time_up_bound}, $t^{\pi_0}_{hit}$ and $t^{\pi_0}_{cov}$ are finite, and, by the definition of the cover time, it immediately follows that $\sum_{k=1}^{K_T} \ee \left [ R_{0,k} \lvert A_2 \right ] \prob \left ( A_2 \right ) \leq t^{\pi_0}_{cov} \sqrt{T}$. Next, by \cite[Theorem 11.2]{levin2017markov}, we have $t^{\pi_0}_{cov} \leq t^{\pi_0}_{hit}(1 + \frac{1}{2} + ... + \frac{1}{S-1})$, which can be further bounded by $t^{\pi_0}_{hit}(\log S + C) \sqrt{T}$ using the harmonic numbers approximation for a relatively small constant, e.g., $C=2$. Finally, by Assumption \ref{assum:WASP}, we observe
    \begin{align*}
        \sum_{k=1}^{K_T} \ee \left [ R_{0,k} \lvert A_2 \right ] \prob \left ( A_2 \right ) \leq t^{\pi_0}_{hit} \left ( \log S + C \right ) \sqrt{T}  \leq H \left ( \log S + C \right ) \sqrt{T}.
    \end{align*}
    For the second term of \eqref{eq:regret_decompose}, conditioned on the good event, $\{ p_k \in \mathcal{G} \}$, the sampled CMDP is feasible, and the standard analysis of \cite{TS_MDP_Ouyang_2017} can be applied. Lemma \ref{lm:regret_on_good_event} shows that this term can be bounded by $(H+1) \sqrt{2SAT \log (T)} + 49HS \sqrt{AT \log (AT)}$.

    Putting both bounds together, we obtain a regret bound of:
    \begin{align*}
        BR_+ (T;c_0) & 
        \leq O \left ( H S \sqrt{AT \log (AT)} + H \sqrt{T} \log (S) \right ).
    \end{align*}
    
    \paragraph{Bounding regret of auxiliary cost components.}

    Without loss of generality, fix the cost component $c_i$ and its threshold $\tau_i$ for some $i$ and focus on analyzing the $i$-th component regret. Similarly to the decomposition of the main component, we obtain:
    \begin{align}
    \label{eq:regret_decompose_aux}
        BR_+ (T;c_i) & 
        = \ee \left [ \sum_{t=0}^T \left ( c_i(s_t, a_t) - \tau_i \right )_+ \right ] 
        = \sum_{k=1}^{K_T} \ee \left [ R_{i,k} \right ] 
        \notag
        \\
        & 
        = \sum_{k=1}^{K_T} \ee \left [ R_{i,k} \lvert p_k \notin \mathcal{G} \right ] \prob \left ( p_k \notin \mathcal{G} \right ) 
        + \sum_{k=1}^{K_T} \ee \left [ R_{i,k} \lvert p_k \in \mathcal{G} \right ] \prob \left ( p_k \in \mathcal{G} \right )  
        \notag
    \end{align}
    where $R_{i,k} = \sum_{t=t_k}^{t_{k+1}-1} \left [ c_i(s_t, a_t) - (t_{k+1} - t_k)  \tau_i \right ]_+$.

    The first term can be analyzed similarly to the main cost component and bounded by $H (\log S + C) + 1$. The regret bound of the second term is the same as the regret bound of the analogous term of the main cost component. Its analysis is marginally different and provided in Lemma \ref{lm:regret_on_good_event_aux}. $\qed$

\section{Simulation results}
\label{sec:experiments}

In this section, we evaluate the performance of \textsc{PSConRL}. We present \textsc{PSConRL} using Dirichlet priors with parameters $[0.01, \dots, 0.01]$. The Dirichlet distribution is a convenient choice for maintaining posteriors for the transition probability vectors $p(s, a)$ since it is a conjugate prior for categorical and multinomial distributions. Moreover, Dirichlet prior is proven to be highly effective for any underlying MDP in unconstrained problems \cite{whyPSbetter}. 

We consider three algorithms as baselines: \textsc{ConRL} \citep{NEURIPS2020_Brantley}, \textsc{C-UCRL} \citep{pmlr-v120-zheng20a}, and \textsc{UCRL-CMDP} \citep{Singh_CMDP_2020}. \textsc{ConRL} and \textsc{UCRL-CMDP} are OFU algorithms developed for finite- and infinite-horizon setting correspondingly. \textsc{C-UCRL} considers conservative (safe) exploration in the infinite-horizon setting by following a principle of ``pessimism in the face of cost uncertainty''. We run our experiments on three gridworld environments: Marsrover 4x4, Marsrover 8x8 \cite{pmlr-v120-zheng20a}, and Box \cite{Leike_aisafegrid_2017}. To enable fair comparison, all algorithms were extended to the unknown reward/costs and unknown probability transitions setting (see Appendix \ref{apx:experiments} for more experimental details). 

Figure \ref{fig:sum_results} illustrates the simulation results of \textsc{PSConRL}, \textsc{ConRL}, \textsc{C-UCRL}, and \textsc{UCRL-CMDP} algorithms across three benchmark environments. Due to its computational inefficiency \textsc{UCRL-CMDP} is implemented only for the Marsrover 4x4 environment. The top row shows the cumulative regret of the main cost component. The bottom row presents the cumulative constraint violation. 

We first analyze the behavior of the algorithm on Marsrover environments (left and middle columns). The cumulative regret (top row) shows that \textsc{PSConRL} outperforms all three algorithms, except for \textsc{ConRL} for Marsrover 4x4 environment, where both algorithms demonstrate indistinguishable performance. Looking at the cumulative constraint violation (bottom row), we see that \textsc{PSConRL} is comparable with \textsc{C-UCRL}, the only algorithm that addresses safe exploration. In the Box example (right column), \textsc{ConPSRL} significantly outperforms the OFU algorithms, which incur near-linear regret. We note that exploration is relatively costly in this benchmark compared to Marsrover environments (see the difference on the $x$ and $y$-axes in the top row), which suggests that OFU algorithms might be impractical in (at least some) problems where exploration is non-trivial. In Figure \ref{fig:sum_results_app} (in Appendix \ref{apx:experiments}), we further elaborate on the cost performance of the algorithms interpreting regret behavior.

\begin{figure}[t]
    \centering
    \includegraphics[width=0.99\textwidth]{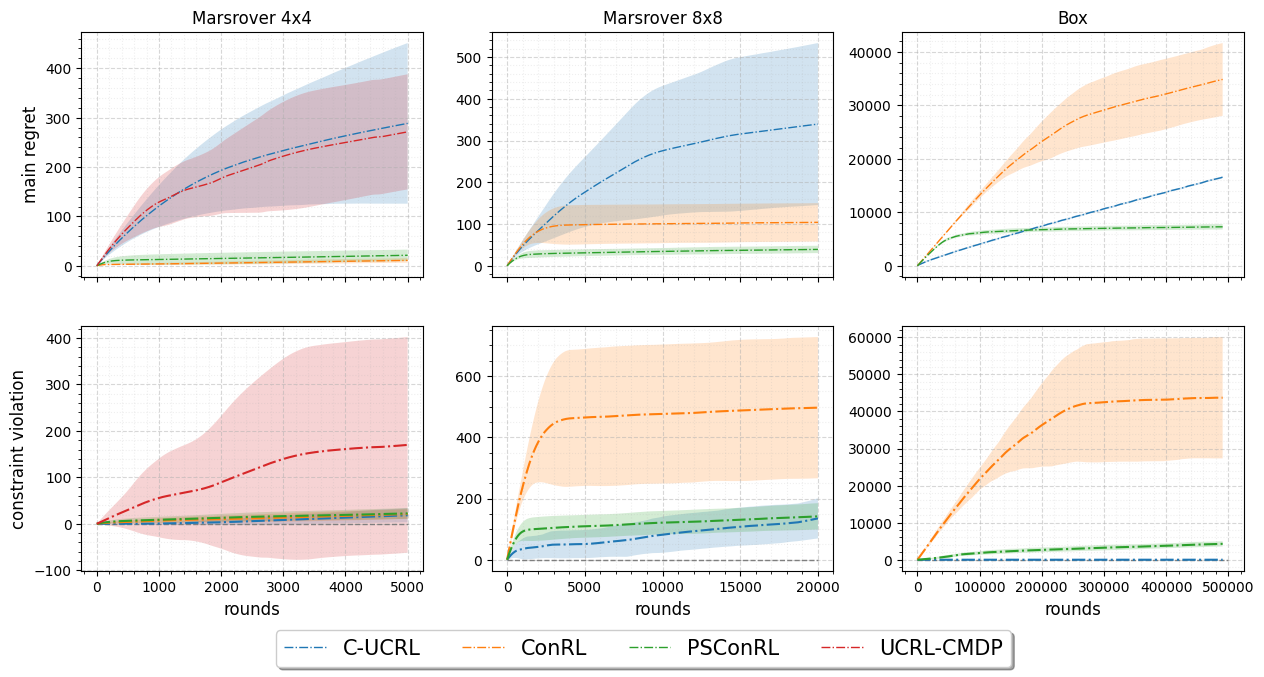}
    \caption{The main regret and constraint violation of the algorithms as a function of the horizon for Marsrover 4x4 (left column), Marsrover 8x8 (middle column), and Box (right column) environments. \textbf{(Top row)} shows the cumulative regret of the main cost component. \textbf{(Bottom row)} shows the cumulative constraint violation. Results are averaged over 100 runs for Marsrover 4x4 and over 30 runs for Marsrover 8x8 and Box.}
    \label{fig:sum_results}
\end{figure}


\section{Related work}
\label{sec:lit_review}

Provably efficient exploration in unknown CMDPs is a recurring theme in reinforcement learning. Numerous algorithms have been discovered for the finite-horizon setting \cite{Efroni_2020_CMDP, NEURIPS2020_Brantley, qiu_2021_cmdp_posterior, Liu_2021}. In the infinite-horizon undiscounted setting, two works \cite{pmlr-v120-zheng20a, Singh_CMDP_2020} that we discussed above propose algorithms based on the OFU principle: with \cite{pmlr-v120-zheng20a} considering safe exploration and establishing $\tilde{O}(SAT^{3/4})$ frequentist regret bound (with no constraint violation) and \cite{Singh_CMDP_2020} establishing $\tilde{O}(D \sqrt{SA} T^{2/3})$ frequentist regret bounds for the main regret and constraints violation, where $D$ is diameter of CMDP $M$. Later, \cite{chen_2022_optimisQlearn} considered optimistic Q-learning providing a tighter frequentist regret bound of $\Tilde{O} (sp(p_{\ast}) S \sqrt{AT})$ for both main regret and constraints violation that strictly improves the result of \cite{Singh_CMDP_2020}, although under the known bias span $sp(p_{\ast})$. Very recently, \cite{Agarwal_2021_PSRL} has presented a regret bound of $\Tilde{O} (H  \sqrt{SAT})$ using posterior sampling for both main regret and constraints violation for the subclass of ergodic CMDPs. While \cite{Agarwal_2021_PSRL} presents promising theoretical results, the setting there appears to be different from ours and does not generalize to communicating CMDPs. Table \ref{tbl:lit_rev_short} summarizes algorithms that address provably efficient exploration in the infinite-horizon undiscounted setting.

\begin{table}[htbp]
  \centering
  \begin{tabular}{ccccccccc}
    \hline
    \multicolumn{1}{c}{Algorithm} && \multicolumn{1}{c}{Main Regret} && \multicolumn{1}{c}{Constraint violation} && \multicolumn{1}{c}{Setting} \\

    \hline

    \textsc{C-UCRL} &&&&&& \\
      
     \cite{pmlr-v120-zheng20a} &&  \multirow{-2}{*}{$\tilde{O}(mSAT^{3/4} )$} && \multirow{-2}{*}{0} && \multirow{-2}{*}{frequentist} & \\
   \cline{0-7}
    
     \textsc{UCRL-CMDP} &  &&&&& \\
     
     \cite{Singh_CMDP_2020} && \multirow{-2}{*}{$\tilde{O}(D\sqrt{SA}T^{2/3} )$} && \multirow{-2}{*}{$\tilde{O}(D\sqrt{SA}T^{2/3} )$} && \multirow{-2}{*}{frequentist} &  \\
   \cline{0-7}

   &&&&&& \\
    \multirow{-2}{*}{\cite{chen_2022_optimisQlearn}} && \multirow{-2}{*}{$\tilde{O}(sp(p_{\ast})S \sqrt{AT} )$} && \multirow{-2}{*}{$\tilde{O}( sp(p_{\ast})S \sqrt{AT} )$} && \multirow{-2}{*}{frequentist}  \\

     \cline{0-7}

     \textsc{CMDP-PSRL} &&&&&& \\
    \cite{Agarwal_2021_PSRL} &&  \multirow{-2}{*}{$\tilde{O}( H \sqrt{SAT})$} && \multirow{-2}{*}{$\tilde{O}( H \sqrt{SAT})$} && \multirow{-2}{*}{unspecified} \\
   
   \cline{0-7}
    \textsc{PSConRL} &&&&&& \\ 
    (this work) &&  \multirow{-2}{*}{$\tilde{O}( HS \sqrt{AT})$} && \multirow{-2}{*}{$\tilde{O}( HS \sqrt{AT})$} && \multirow{-2}{*}{Bayesian} \\
   \hline
  \end{tabular}
  \caption{Summary of work on provably efficient constrained reinforcement learning in the infinite-horizon undiscounted setting. $S$ and $A$ represent number of states and actions, $m$ is the number of constraints, $T$ is the total horizon,  $H$ is the bound of the hitting time, and $D$ is diameter of CMDP.}
  \label{tbl:lit_rev_short}
\end{table}


Among other related work, Lagrangian relaxation is a popular technique for solving CMDPs. The works \cite{Achiam_CMDP2017, tessler2018reward} present constrained policy optimization approaches that have demonstrated prominent successes in artificial environments. However, these approaches are notoriously sample-inefficient and lack theoretical guarantees. More scalable versions of the Lagrangian-based methods have been proposed in \cite{Chow_CMDP2018, qiu_2021_cmdp_posterior, Liu_2021, chen2021primaldual}  (see \cite{ijcai2021p614} for a survey). In general, as discussed in \cite{ijcai2021p614}, the Lagrangian relaxation method can achieve high performance, but this approach is sensitive to the initialization of the Lagrange multipliers and learning rate.

\section{Conclusion}
\label{sec:conclusion}
Our paper has presented a novel algorithm for efficient exploration in constrained reinforcement learning under the infinite-horizon undiscounted average cost criterion. Our \textsc{PSConRL} algorithm achieves near-optimal Bayesian regret bounds for each cost component, filling a gap in the theoretical analysis of posterior sampling for communicating CMDPs. We validate our approach using simulations on three gridworld domains and show that \textsc{PSConRL} quickly converges to the optimal policy and consistently outperforms existing algorithms. Our work represents a solid step toward designing reinforcement learning algorithms for real-world problems.

The feasibility guarantees provided in this work might be of great value for further research in constrained reinforcement learning. In particular, we believe that our theoretical analysis can be extended to the frequentist regret bound by incorporating existing methods such as \cite{NIPS2017_Shipra_OPSRL} or \cite{tiapkin2022optimistic}.

\acks{We thank a bunch of people and funding agency.}

\newpage
\bibliography{main}

\begin{thebibliography}{33}
\providecommand{\natexlab}[1]{#1}
\providecommand{\url}[1]{\texttt{#1}}
\expandafter\ifx\csname urlstyle\endcsname\relax
  \providecommand{\doi}[1]{doi: #1}\else
  \providecommand{\doi}{doi: \begingroup \urlstyle{rm}\Url}\fi

\bibitem[Abbasi-Yadkori and Szepesv\'{a}ri(2015)]{AY_bayesian_control_2015}
Yasin Abbasi-Yadkori and Csaba Szepesv\'{a}ri.
\newblock Bayesian optimal control of smoothly parameterized systems.
\newblock In \emph{Proceedings of the Thirty-First Conference on Uncertainty in Artificial Intelligence}, UAI'15, 2015.

\bibitem[Achiam et~al.(2017)Achiam, Held, Tamar, and Abbeel]{Achiam_CMDP2017}
Joshua Achiam, David Held, Aviv Tamar, and Pieter Abbeel.
\newblock Constrained policy optimization.
\newblock In \emph{Proceedings of the 34th International Conference on Machine Learning - Volume 70}, ICML'17. JMLR.org, 2017.

\bibitem[Afsar et~al.(2021)Afsar, Crump, and Far]{asfar2021_RLforRS:survey}
M.~Mehdi Afsar, Trafford Crump, and Behrouz Far.
\newblock Reinforcement learning based recommender systems: A survey, 2021.

\bibitem[Agarwal et~al.(2022)Agarwal, Bai, and Aggarwal]{Agarwal_2021_PSRL}
Mridul Agarwal, Qinbo Bai, and Vaneet Aggarwal.
\newblock Regret guarantees for model-based reinforcement learning with long-term average constraints.
\newblock In \emph{Proceedings of the Thirty-Eighth Conference on Uncertainty in Artificial Intelligence}, 2022.

\bibitem[Agrawal and Jia(2017)]{NIPS2017_Shipra_OPSRL}
Shipra Agrawal and Randy Jia.
\newblock Optimistic posterior sampling for reinforcement learning: worst-case regret bounds.
\newblock In \emph{Advances in Neural Information Processing Systems}, 2017.

\bibitem[Altman(1999)]{Altman99constrainedmarkov}
Eitan Altman.
\newblock Constrained markov decision processes, 1999.

\bibitem[Bartlett and Tewari(2009)]{regal_2009}
Peter~L. Bartlett and Ambuj Tewari.
\newblock Regal: A regularization based algorithm for reinforcement learning in weakly communicating mdps.
\newblock In \emph{Proceedings of the Twenty-Fifth Conference on Uncertainty in Artificial Intelligence}, UAI '09, page 35–42, Arlington, Virginia, USA, 2009. AUAI Press.
\newblock ISBN 9780974903958.

\bibitem[Brantley et~al.(2020)Brantley, Dudik, Lykouris, Miryoosefi, Simchowitz, Slivkins, and Sun]{NEURIPS2020_Brantley}
Kiant\'{e} Brantley, Miro Dudik, Thodoris Lykouris, Sobhan Miryoosefi, Max Simchowitz, Aleksandrs Slivkins, and Wen Sun.
\newblock Constrained episodic reinforcement learning in concave-convex and knapsack settings.
\newblock In \emph{Advances in Neural Information Processing Systems}, 2020.

\bibitem[Chen et~al.(2022)Chen, Jain, and Luo]{chen_2022_optimisQlearn}
Liyu Chen, Rahul Jain, and Haipeng Luo.
\newblock Learning infinite-horizon average-reward markov decision processes with constraints, 2022.

\bibitem[Chen et~al.(2021)Chen, Dong, and Wang]{chen2021primaldual}
Yi~Chen, Jing Dong, and Zhaoran Wang.
\newblock A primal-dual approach to constrained markov decision processes, 2021.

\bibitem[Chow et~al.(2018)Chow, Nachum, Duenez-Guzman, and Ghavamzadeh]{Chow_CMDP2018}
Yinlam Chow, Ofir Nachum, Edgar Duenez-Guzman, and Mohammad Ghavamzadeh.
\newblock A lyapunov-based approach to safe reinforcement learning.
\newblock In \emph{Proceedings of the 32nd International Conference on Neural Information Processing Systems}, NIPS'18. Curran Associates Inc., 2018.

\bibitem[Ding et~al.(2021)Ding, Wei, Yang, Wang, and Jovanovic]{DingWYWJ21}
Dongsheng Ding, Xiaohan Wei, Zhuoran Yang, Zhaoran Wang, and Mihailo~R. Jovanovic.
\newblock Provably efficient safe exploration via primal-dual policy optimization.
\newblock In \emph{The 24th International Conference on Artificial Intelligence and Statistics, AISTATS 2021, April 13-15, 2021, Virtual Event}, 2021.

\bibitem[Efroni et~al.(2020)Efroni, Mannor, and Pirotta]{Efroni_2020_CMDP}
Yonathan Efroni, Shie Mannor, and Matteo Pirotta.
\newblock Exploration-exploitation in constrained mdps, 2020.

\bibitem[Jafarnia-Jahromi et~al.(2021)Jafarnia-Jahromi, Chen, Jain, and Luo]{jafarniajahromi2021online}
Mehdi Jafarnia-Jahromi, Liyu Chen, Rahul Jain, and Haipeng Luo.
\newblock Online learning for stochastic shortest path model via posterior sampling, 2021.

\bibitem[Jaksch et~al.(2010)Jaksch, Ortner, and Auer]{JMLR:v11:jaksch10a}
Thomas Jaksch, Ronald Ortner, and Peter Auer.
\newblock Near-optimal regret bounds for reinforcement learning.
\newblock \emph{Journal of Machine Learning Research}, 2010.

\bibitem[Kalagarla et~al.(2023)Kalagarla, Jain, and Nuzzo]{kalagarla2023safe}
Krishna~C Kalagarla, Rahul Jain, and Pierluigi Nuzzo.
\newblock Safe posterior sampling for constrained mdps with bounded constraint violation, 2023.

\bibitem[Lai and Robbins(1985)]{LAI19854}
T.L Lai and Herbert Robbins.
\newblock Asymptotically efficient adaptive allocation rules.
\newblock \emph{Advances in Applied Mathematics}, 6\penalty0 (1):\penalty0 4--22, 1985.
\newblock ISSN 0196-8858.
\newblock \doi{https://doi.org/10.1016/0196-8858(85)90002-8}.
\newblock URL \url{https://www.sciencedirect.com/science/article/pii/0196885885900028}.

\bibitem[Lattimore and Szepesvári(2020)]{lattimore_szepesvári_2020}
Tor Lattimore and Csaba Szepesvári.
\newblock \emph{Bandit Algorithms}.
\newblock Cambridge University Press, 2020.
\newblock \doi{10.1017/9781108571401}.

\bibitem[Le et~al.(2019)Le, Voloshin, and Yue]{Le2019BatchPL}
Hoang~Minh Le, Cameron Voloshin, and Yisong Yue.
\newblock Batch policy learning under constraints.
\newblock \emph{ArXiv}, abs/1903.08738, 2019.

\bibitem[Leike et~al.(2017)Leike, Martic, Krakovna, Ortega, Everitt, Lefrancq, Orseau, and Legg]{Leike_aisafegrid_2017}
Jan Leike, Miljan Martic, Victoria Krakovna, Pedro~A. Ortega, Tom Everitt, Andrew Lefrancq, Laurent Orseau, and Shane Legg.
\newblock Ai safety gridworlds, 2017.
\newblock URL \url{https://arxiv.org/abs/1711.09883}.

\bibitem[Levin and Peres(2017)]{levin2017markov}
David~A Levin and Yuval Peres.
\newblock \emph{Markov chains and mixing times}, volume 107.
\newblock American Mathematical Soc., 2017.

\bibitem[Liu et~al.(2021{\natexlab{a}})Liu, Zhou, Kalathil, Kumar, and Tian]{Liu_2021}
Tao Liu, Ruida Zhou, Dileep Kalathil, P.~R. Kumar, and Chao Tian.
\newblock Learning policies with zero or bounded constraint violation for constrained mdps, 2021{\natexlab{a}}.
\newblock URL \url{https://arxiv.org/abs/2106.02684}.

\bibitem[Liu et~al.(2021{\natexlab{b}})Liu, Halev, and Liu]{ijcai2021p614}
Yongshuai Liu, Avishai Halev, and Xin Liu.
\newblock Policy learning with constraints in model-free reinforcement learning: A survey.
\newblock In Zhi-Hua Zhou, editor, \emph{Proceedings of the Thirtieth International Joint Conference on Artificial Intelligence, {IJCAI-21}}, pages 4508--4515. International Joint Conferences on Artificial Intelligence Organization, 8 2021{\natexlab{b}}.
\newblock \doi{10.24963/ijcai.2021/614}.
\newblock URL \url{https://doi.org/10.24963/ijcai.2021/614}.
\newblock Survey Track.

\bibitem[Osband and Van~Roy(2017)]{whyPSbetter}
Ian Osband and Benjamin Van~Roy.
\newblock Why is posterior sampling better than optimism for reinforcement learning?
\newblock In \emph{Proceedings of the 34th International Conference on Machine Learning - Volume 70}, ICML'17, page 2701–2710. JMLR.org, 2017.

\bibitem[Osband et~al.(2013)Osband, Russo, and Van~Roy]{Osband_PSRL2013}
Ian Osband, Daniel Russo, and Benjamin Van~Roy.
\newblock (more) efficient reinforcement learning via posterior sampling, 2013.
\newblock URL \url{https://arxiv.org/abs/1306.0940}.

\bibitem[Ouyang et~al.(2017)Ouyang, Gagrani, Nayyar, and Jain]{TS_MDP_Ouyang_2017}
Yi~Ouyang, Mukul Gagrani, Ashutosh Nayyar, and Rahul Jain.
\newblock Learning unknown markov decision processes: A thompson sampling approach.
\newblock In \emph{Proceedings of the 31st International Conference on Neural Information Processing Systems}, NIPS'17, 2017.

\bibitem[Puterman(1994)]{Puterman_mdp}
Martin~L. Puterman.
\newblock \emph{Markov Decision Processes: Discrete Stochastic Dynamic Programming}.
\newblock John Wiley \& Sons, Inc., USA, 1st edition, 1994.
\newblock ISBN 0471619779.

\bibitem[Qiu et~al.(2020)Qiu, Wei, Yang, Ye, and Wang]{qiu_2021_cmdp_posterior}
Shuang Qiu, Xiaohan Wei, Zhuoran Yang, Jieping Ye, and Zhaoran Wang.
\newblock Upper confidence primal-dual reinforcement learning for cmdp with adversarial loss, 2020.
\newblock URL \url{https://arxiv.org/abs/2003.00660}.

\bibitem[Singh et~al.(2020)Singh, Gupta, and Shroff]{Singh_CMDP_2020}
Rahul Singh, Abhishek Gupta, and Ness~B. Shroff.
\newblock Learning in markov decision processes under constraints.
\newblock \emph{CoRR}, abs/2002.12435, 2020.

\bibitem[Tessler et~al.(2019)Tessler, Mankowitz, and Mannor]{tessler2018reward}
Chen Tessler, Daniel~J. Mankowitz, and Shie Mannor.
\newblock Reward constrained policy optimization.
\newblock In \emph{International Conference on Learning Representations}, 2019.

\bibitem[Thompson(1933)]{THOMPSON_1933}
William~R Thompson.
\newblock {On the likelihood that one unknown probability exceeds another in view of the evidence of two samples}.
\newblock \emph{Biometrika}, 25\penalty0 (3-4):\penalty0 285--294, 12 1933.

\bibitem[Tiapkin et~al.(2022)Tiapkin, Belomestny, Calandriello, Moulines, Munos, Naumov, Rowland, Valko, and MENARD]{tiapkin2022optimistic}
Daniil Tiapkin, Denis Belomestny, Daniele Calandriello, Eric Moulines, Remi Munos, Alexey Naumov, Mark Rowland, Michal Valko, and Pierre MENARD.
\newblock Optimistic posterior sampling for reinforcement learning with few samples and tight guarantees.
\newblock In Alice~H. Oh, Alekh Agarwal, Danielle Belgrave, and Kyunghyun Cho, editors, \emph{Advances in Neural Information Processing Systems}, 2022.

\bibitem[Zheng and Ratliff(2020)]{pmlr-v120-zheng20a}
Liyuan Zheng and Lillian Ratliff.
\newblock Constrained upper confidence reinforcement learning.
\newblock In \emph{Proceedings of the 2nd Conference on Learning for Dynamics and Control}, 2020.

\end{thebibliography}

\newpage
\appendix


\section{Omitted details for Section \ref{sec:regret_bound}}
\subsection{Proof of Feasibility lemma (Lemma \ref{lm:feasibility})}
\label{apx:proof_feasibility}

\begin{proof}
    Fix some $i \in \{1, \dots, m\}$. Further, we will omit index $i$ and write $c$ and $\tau$ instead of $c_i$ and $\tau_i$. With slight abuse of notation, we rewrite the equation \eqref{eq:bellman} in vector form:
    \begin{equation}
        \label{eq:bellman_vector}
        J^{\pi, p_{\ast}, c} + v^{\pi, p_{\ast}}_s = c_{s,a} + \left ( p^{\ast}_{s,a} \right )^{\intercal} v^{\pi, p_{\ast}},
    \end{equation}
    where $J^{\pi, p_{\ast}, c} = J^{\pi}(c; p_{\ast})$, $c(s,a) = c_{s,a}$, $v^{\pi, p_{\ast}}_s=v^{\pi}(s; p_{\ast})$, and $p^{\ast}_{s,a} = p_{\ast}(\cdot | s, a)$. 

    Let $P_{\bar{\pi}}^k$ be the transition matrix whose rows are formed by the vectors $p^k_{s,\bar{\pi}(s)}$, where $p^{k}_{s,a} = p_{k}(\cdot | s, a)$, and $P^{\ast}_{\bar{\pi}}$ be the transition matrix whose rows are formed by the vectors $p^{\ast}_{s,\bar{\pi}(s)}$. Since $\norm{p_k(\cdot|s,a) - p_{\ast}(\cdot|s,a)}_1 \leq \sqrt{\frac{14S \log (2A T t_k)}{\max \{1, N_{t_k}(s,a)\}}}$, $N_{t_k}(s,a) \geq \sqrt{T}$ for all $(s,a)$,  and  by Assumption \ref{assum:WASP} the span of the bias function $v^{\bar{\pi}, p_{\ast}}$ is at most $H$, we observe
    \begin{equation*}
        \left ( p_k(\cdot|s,a) - p_{\ast}(\cdot|s,a) \right )^\intercal v^{\bar{\pi}, p_{\ast}} \leq \norm{p_k(\cdot|s,a) - p_{\ast}(\cdot|s,a)}_1 \norm{v^{\bar{\pi}, p_{\ast}} }_\infty \leq \delta H
    \end{equation*}
    where $\delta = \sqrt{\frac{14S \log (2A T t_k)}{\sqrt{T}}}$. Above implies 
    \begin{equation}
        \label{eq:lim_matrices_diff}
        \left( P_{\bar{\pi}}^k - P^{\ast}_{\bar{\pi}} \right ) v^{\bar{\pi}, p_{\ast}} \leq \delta H \textbf{1}
    \end{equation}
    where $\textbf{1}$ is the vector of all 1s.
    
    Following \cite{NIPS2017_Shipra_OPSRL}, let $(P_{\bar{\pi}}^k)^{\ast}$ denote the limiting matrix for Markov chain with transition matrix $P_{\bar{\pi}}^k$. Observe that $P_{\bar{\pi}}^k$ is aperiodic and irreducible because of Assumption \ref{assum:slater}. This implies that $(P_{\bar{\pi}}^k)^{\ast}$ is of the form $\textbf{1}\boldsymbol{q}^{\intercal}$ where $\boldsymbol{q}$ is the stationary distribution of $P_{\bar{\pi}}^k$ (refer to (A.4) in \cite{Puterman_mdp}). Also, $(P_{\bar{\pi}}^k)^{\ast} P_{\bar{\pi}}^k = (P_{\bar{\pi}}^k)^{\ast}$ and $(P_{\bar{\pi}}^k)^{\ast} \textbf{1} = \textbf{1}$.

    Therefore, the gain of policy $\bar{\pi}$
    \begin{equation*}
        J^{\bar{\pi}, p_k, c} \textbf{1} = (c_{\bar{\pi}}^{\intercal} \boldsymbol{q}) \textbf{1} = (P_{\bar{\pi}}^k)^{\ast} c_{\bar{\pi}}
    \end{equation*}
    where $c_{\bar{\pi}}$ is the S dimensional vector $\left [ c_{s, \bar{\pi}(s)} \right ]_{s=1,\dots,S}$. Now
    \begin{align*}
        J^{\bar{\pi}, p_k, c} \textbf{1} - J^{\bar{\pi}, p_{\ast}, c} \textbf{1} & = (P_{\bar{\pi}}^k)^{\ast} c_{\bar{\pi}} - J^{\bar{\pi}, p_{\ast}, c} \textbf{1} \\
        & = (P_{\bar{\pi}}^k)^{\ast} c_{\bar{\pi}} - J^{\bar{\pi}, p_{\ast}, c} \left( (P_{\bar{\pi}}^k)^{\ast} \textbf{1} \right ) 
        && \quad \tag{using $(P_{\bar{\pi}}^k)^{\ast} \textbf{1} = \textbf{1}$} \\
        & = (P_{\bar{\pi}}^k)^{\ast} \left( c_{\bar{\pi}} - J^{\bar{\pi}, p_{\ast}, c}  \textbf{1} \right ) \\
        & = (P_{\bar{\pi}}^k)^{\ast} \left( I - P^{\ast}_{\bar{\pi}} \right ) v^{\bar{\pi}, p_{\ast}}
        && \quad \tag{using \eqref{eq:bellman_vector}} \\
        & = (P_{\bar{\pi}}^k)^{\ast} \left( P_{\bar{\pi}}^k - P^{\ast}_{\bar{\pi}} \right ) v^{\bar{\pi}, p_{\ast}}
        && \quad \tag{using $(P_{\bar{\pi}}^k)^{\ast} P_{\bar{\pi}}^k = (P_{\bar{\pi}}^k)^{\ast}$} \\
        & \leq H \delta \textbf{1}
        && \quad \tag{using \eqref{eq:lim_matrices_diff} and $(P_{\bar{\pi}}^k)^{\ast} \textbf{1} = \textbf{1}$}
    \end{align*}

    \noindent We finish proof by observing that $H \delta \leq \gamma$. Thus, 
    \begin{equation*}
        J^{\bar{\pi}, p_k, c} - J^{\bar{\pi}, p_{\ast}, c} \leq H \delta \leq \gamma.
    \end{equation*}
\end{proof}

\subsection{Proof of Lemma \ref{lm:cover_time_up_bound}}
\label{apx:upper_bound_cover_time}

\begin{proof}
    By \cite[Proposition 8.3.1]{Puterman_mdp}, for any communicating CMDP, there exists a policy $\tilde{\pi}$ which induces an ergodic Markov chain. We show that uniform random policy $\pi_0$ also induces an ergodic Markov chain.

    Let $P_{\tilde{\pi}}$ and $P_{\pi_0}$ be the transition matrices for policies $\tilde{\pi}$ and $\pi_0$ with elements $\tilde{p}_{ij}$ and $p_{ij}$, correspondingly, $1 \leq i,j \leq S$. Note that every nonzero element in $P_{\tilde{\pi}}$ is also nonzero in $P_{\pi_0}$ because $\pi_0$ assigns a nonzero probability to every action that has nonzero probability in $\tilde{\pi}$, and other elements are non-negative. Assume that there exist two states $s_i, s_j$ such that $p_{ij}^t = 0$ for any $t > 0$. Then, $\tilde{p}_{ij}^t = 0$ for any $t > 0$, which contradicts to the ergodicity of $P_{\tilde{\pi}}$. Therefore, for any two states $s_i, s_j$ there exists finite $t$ such that $p_{ij}^t > 0$. Thus, $P_{\pi_0}$ corresponds to ergodic chain, and $t^{\pi_0}_{hit}$ is finite as a hitting time of ergodic CMDP, where we recall that $t^{\pi_0}_{hit} = \max_{s, s' \in \mathcal{S}} \mathbb{E}^{\pi_0}_{p}\tau_{ss'}$.
\end{proof}

\subsection{Regret of the main cost on the good event}
\label{apx:regret_on_the_good_event}
\begin{lemma}[Adapted from Theorem 1 of \cite{TS_MDP_Ouyang_2017}]
    \label{lm:regret_on_good_event}
    Under Assumption \ref{assum:WASP}, conditioned on the good event $\{ p_k \in \mathcal{G} \}$,
    \begin{equation*}
        \sum_{k=1}^{K_T} \ee \left [ R_{0,k} \lvert p_k \in \mathcal{G} \right ] \prob \left ( p_k \in \mathcal{G} \right ) \leq (H+1) \sqrt{2SAT \log (T)} + 49HS \sqrt{AT \log (AT)}.
    \end{equation*}
\end{lemma}

Most of the analysis here recovers the analysis of \cite{TS_MDP_Ouyang_2017}. Nonetheless, for the sake of clarity, we provide the complete proof of Lemma \ref{lm:regret_on_good_event}.

\begin{proof}
    Conditioned on the good event $\{ p_k \in \mathcal{G} \}$, every policy $\pi_k$ for $k=1, \dots, K_T$ of Algorithm \ref{alg1:psrl_transitions} is well defined. Therefore, we can apply the Bellman equation \eqref{eq:bellman} to $c_0(s_t,a_t)$, and decompose $R_{0,k}$ into the following terms.
    \begin{align}
        & \sum_{k=1}^{K_T} \ee \left [ R_{0,k} \lvert p_k \in \mathcal{G} \right ]  \prob \left ( p_k \in \mathcal{G} \right ) \leq \sum_{k=1}^{K_T} \ee \left [ R_{0,k} \lvert p_k \in \mathcal{G} \right ]
        =  
        \sum_{k=1}^{K_T} \ee \left [ \sum_{t=t_k}^{t_{k+1}-1} \big( c_0(s_t,a_t)   -  J^{\pi_{\ast}}(c_0; p_{\ast}) \big) \right ] 
        \notag\\
        & = \sum_{k=1}^{K_T} \ee \left [ \sum_{t=t_k}^{t_{k+1}-1} \left( J^{\pi_k}(c_0; p_k)   -  J^{\pi_{\ast}}(c_0; p_{\ast}) +  v_{\pi_k}(s_t,p_k) -  \sum_{s' \in \mathcal S}p_k(s'|s_t,a_t)v_{\pi_k}(s',p_k)
        \right)\right]
        \notag\\
        & =
        \sum_{k=1}^{K_T} \ee \left [ \sum_{t=t_k}^{t_{k+1}-1} \big( J^{\pi_k}(c_0; p_k) -  J^{\pi_{\ast}}(c_0; p_{\ast}) \big) \right ]
        +
        \sum_{k=1}^{K_T} \ee \left[\sum_{t=t_k}^{t_{k+1}-1}  \left[  v_{\pi_k}(s_t,p_k) -  v_{\pi_k}(s_{t+1},p_k)
        \right]\right]
        \notag\\
        & +
        \sum_{k=1}^{K_T} \ee \left[\sum_{t=t_k}^{t_{k+1}-1}  \left[ v_{\pi_k}(s_{t+1},p_k) - \sum_{s' \in \mathcal S}p(s'|s_t,a_t)v_{\pi_k}(s',p_k)
        \right]\right]
        = R_0+R_1+R_2.
        \notag
        \label{eq:costregretbound}
    \end{align}
    Next, applying lemmas \ref{lm:R0}, \ref{lm:R1}, \ref{lm:R2} to $R_0$, $R_1$, $R_2$ correspondingly, gives us the result.
\end{proof}

\subsection{Regret of the auxiliary costs on the good event}
\label{apx:regret_on_the_good_event_aux}

\begin{lemma}
    \label{lm:regret_on_good_event_aux}
    Under Assumption \ref{assum:WASP}, conditioned on the good event $\{ p_k \in \mathcal{G} \}$,
    \begin{equation*}
        \sum_{k=1}^{K_T} \ee \left [ R_{i,k} \lvert p_k \in \mathcal{G} \right ] \prob \left ( p_k \in \mathcal{G} \right ) \leq (H+1) \sqrt{2SAT \log (T)} + 49HS \sqrt{AT \log (AT)}.
    \end{equation*}
\end{lemma}

\begin{proof}
    Similarly to lemma \ref{lm:regret_on_good_event}, conditioned on the good event $\{ p_k \in \mathcal{G} \}$, we can decompose $R_{i,k}$ as 
    \begin{align}
        & \sum_{k=1}^{K_T} \ee \left [ R_{i,k} \lvert p_k \in \mathcal{G} \right ]  \prob \left ( p_k \in \mathcal{G} \right ) \leq \sum_{k=1}^{K_T} \ee \left [ R_{i,k} \lvert p_k \in \mathcal{G} \right ]
        =  
        \sum_{k=1}^{K_T} \ee \left [ \sum_{t=t_k}^{t_{k+1}-1} \big( c_i(s_t,a_t)   -  \tau_i \big) \right ] 
        \notag\\
        & = \sum_{k=1}^{K_T} \ee \left [ \sum_{t=t_k}^{t_{k+1}-1} \left( J^{\pi_k}(c_i; p_k)   -  \tau_i +  v_{\pi_k}(s_t,p_k) -  \sum_{s' \in \mathcal S}p_k(s'|s_t,a_t)v_{\pi_k}(s',p_k)
        \right)\right]
        \notag\\
        & =
        \sum_{k=1}^{K_T} \ee \left [ \sum_{t=t_k}^{t_{k+1}-1} \big( J^{\pi_k}(c_0; p_k) -  \tau_i \big) \right ]
        +
        \sum_{k=1}^{K_T} \ee \left[\sum_{t=t_k}^{t_{k+1}-1}  \left[  v_{\pi_k}(s_t,p_k) -  v_{\pi_k}(s_{t+1},p_k)
        \right]\right]
        \notag\\
        & +
        \sum_{k=1}^{K_T} \ee \left[\sum_{t=t_k}^{t_{k+1}-1}  \left[ v_{\pi_k}(s_{t+1},p_k) - \sum_{s' \in \mathcal S}p(s'|s_t,a_t)v_{\pi_k}(s',p_k)
        \right]\right]
        = R_0+R_1+R_2.
        \notag
    \end{align}
    \noindent Next, we note that $\left( J^{\pi_k}(c_i; p_k) - \tau_i \right)$ is negative on the good event $\{ p_k \in \mathcal{G} \}$ for all $k$, and term $R_0$ can be dismissed. $R_1$ and $ R_2$ regret terms can be bounded by lemmas \ref{lm:R1} and \ref{lm:R2} correspondingly.
\end{proof}

\subsection{Auxiliary lemmas}

\begin{lemma}[Lemma 3 from \cite{TS_MDP_Ouyang_2017}]
For any cost function $c : \mathcal{S} \times \mathcal{A} \xrightarrow{} [0, 1]$, 
\label{lm:R0}
\begin{align*}
    \ee \left [ \sum_{k=1}^{K_T}  \sum_{t=t_k}^{t_{k+1}-1}  \big ( J^{\pi_k}(c;p_k) - J^{\pi_{\ast}}(c;p_{\ast}) \big ) \right ]  \leq K_T \leq \sqrt{2SAT \log(T)}.
\end{align*}
\end{lemma}

\begin{lemma}[Lemma 4 from \cite{TS_MDP_Ouyang_2017}]
\label{lm:R1}
\begin{align*}
    \ee \left [ \sum_{k=1}^{K_T} \sum_{t=t_k}^{t_{k+1}-1} \big(v_{\pi_k}(s_{t}, p_k) - v_{\pi_k}(s_{t+1}, p_k)\big) \right ]  \leq HK_T \leq H \sqrt{2SAT \log(T)}.
\end{align*}
\end{lemma}

\begin{lemma}[Lemma 5 from \cite{TS_MDP_Ouyang_2017}]
\label{lm:R2}
\begin{align*}
    \ee\left [\sum_{k=1}^{K_T}\sum_{t=t_k}^{t_{k+1}-1}  \big( v_{\pi_k}(s_{t+1}, p_k) -\sum_{s' \in \mathcal S}p_k(s'|s_t, a_t)v_{\pi_k}(s',p_k)\big)\right ] \leq 49 H S\sqrt{AT\log(AT)}.
\end{align*}
\end{lemma}

\begin{lemma}[Lemma 17 from \cite{JMLR:v11:jaksch10a}]
\label{lm:jaksch_ci}
For any $t \geq 1$, the probability that the true MDP $M$ is not contained in the set of plausible MDPs $\mathcal{M}(t)$ defined as 
\begin{equation*}
    \mathcal{M}(t) = \left \{ (\mathcal{S}, \mathcal{A}, p^{\prime}, \textit{\textbf{c}}, \tau, \rho) : \norm{p^{\prime}(\cdot|s,a) - p_k(\cdot|s,a)}_1 \leq \sqrt{\frac{14S \log (2A t_k / \delta)}{\max \{1, N_{t_k}(s,a)\} }} \right \}
\end{equation*} 
at time $t$ is at most $\frac{\delta}{15t}$. That is $\prob \left \{ M \notin \mathcal{M}(t) \right \} < \frac{\delta}{15t^6}$.
\end{lemma}

\section{Experimental details}
\label{apx:experiments}
\subsection{Baselines: OFU-based algorithms}
\label{apx:benchmarks}

We use three OFU-based algorithms from the existing literature for comparison: \textsc{ConRL} \citep{NEURIPS2020_Brantley}, \textsc{C-UCRL} \citep{pmlr-v120-zheng20a}, and \textsc{UCRL-CMDP} \citep{Singh_CMDP_2020}. 
These algorithms rely on the knowledge of different CMDP components, e.g., \textsc{UCRL-CMDP} relies on knowledge of rewards $r$, whereas \textsc{C-UCRL} uses the knowledge of transitions $p$. To enable fair comparison, all algorithms were extended to the unknown reward/costs and unknown probability transitions setting. Specifically, we assume that each algorithm knows only the states space $\mathcal{S}$ and the action space $\mathcal{A}$, substituting the unknown elements with their empirical estimates:
\begin{equation}
    \label{mean_rewards}
    \Bar{r}_t(s,a) = \frac{\sum_{j=1}^{t-1} \mathbb{I} \{ s_t=s, a_t=a \} r_t}{N_t(s,a) \lor 1}, \quad \forall s \in \mathcal{S}, a \in \mathcal{A},
\end{equation}
\begin{equation}
    \label{mean_costs}
    \Bar{c}_{i,t}(s,a) = \frac{\sum_{j=1}^{t-1} \mathbb{I} \{ s_t=s, a_t=a \} c_{i,t}}{N_t(s,a) \lor 1}, \quad \forall s \in \mathcal{S}, a \in \mathcal{A}, \quad i = 1 \dots, m,
\end{equation}
\begin{equation}
    \label{mean_transitions}
    \Bar{p}_t(s,a,s') = \frac{N_t(s,a,s')}{N_t(s,a) \lor 1}, \quad \forall s, s' \in \mathcal{S}, a \in \mathcal{A}.
\end{equation}
where $r$ is the reward function (inverse main cost $c_0$) and $N_t(s,a)$ and $N_t(s,a,s')$ denote the number of visits to $(s,a)$ and $(s,a,s')$ respectively.

Further, we provide algorithmic-specific details separately for each baseline:
\begin{enumerate}
    \item \textsc{ConRL} implements the principle of optimism under uncertainty by introducing a bonus term $b_t(s,a)$ that favors under-explored actions with respect to each component of the reward vector. In the original work \cite{NEURIPS2020_Brantley}, the authors consider an episodic problem; they add a bonus to the empirical rewards (\ref{mean_rewards}) and subtract it from the empirical costs (\ref{mean_costs}):
    \begin{align*}
        \hat{r}_t(s,a) =\Bar{r}_t(s,a) + b_t(s,a)
            \quad\mathrm{and}\quad
        \hat{c}_t(s,a) =\Bar{r}_t(s,a) - b_t(s,a).
    \end{align*}
    We follow the same principle but recast the problem to the infinite-horizon setting by using the doubling epoch framework described in \cite{JMLR:v11:jaksch10a}.
    
    \item \textsc{C-UCRL} follows a principle of “optimism in the face of reward uncertainty; pessimism in the face of cost uncertainty.” This algorithm, 
    which was developed in \cite{pmlr-v120-zheng20a}, considers conservative (safe) exploration by overestimating both rewards and costs:
    \begin{align*}
        \hat{r}_t(s,a) =\Bar{r}_t(s,a) + b_t(s,a)
            \quad\mathrm{and}\quad
        \hat{c}_t(s,a) =\Bar{r}_t(s,a) + b_t(s,a).
    \end{align*}
    \textsc{C-UCRL} proceeds in episodes of linearly increasing number of rounds $kh$, where $k$ is the episode index and $h$ is the fixed duration given as an input. In each epoch, the random policy \footnote{Original algorithm utilizes a safe baseline during the first $h$ rounds in each epoch, which is assumed to be known. However, to make the comparison as fair as possible, we assume that a random policy is applied instead.} is executed for $h$ steps for additional exploration, and then policy $\pi_k$ is applied for $(k-1)h$ number of steps, making $kh$ the total duration of episode $k$. 
    
    \item Unlike the previous two algorithms, where uncertainty was taken into account by enhancing rewards and costs, \textsc{UCRL-CMDP} \cite{Singh_CMDP_2020} constructs confidence set $\mathcal{C}_t$ over $\Bar{p}_t$:
    \begin{equation*}
        \mathcal{C}_t = \left \{ p: |p(s,a,s') - \Bar{p}_t(s,a,s')| \leq b_t(s,a) \quad \forall (s,a) \right \}.
    \end{equation*}
    
    \textsc{UCRL-CMDP} algorithm proceeds in episodes of fixed duration of $\ceil*{T^{\alpha}}$, where $\alpha$ is an input of the algorithm. At the beginning of each round, the agent solves the following constrained optimization problem in which the decision variables are (i) Occupation measure $\mu(s,a)$, and (ii) “Candidate” transition $p'$:
    \begin{align}
        \max_{\mu, p' \in \mathcal{C}_t} \sum_{s,a} \mu(s,a) r(s,a), \label{eq1'}\\
        \mathrm{s.t.}\quad \sum_{s,a} \mu(s,a) c_i(s,a) \leq \tau_i,\, \quad i=1,\dots,m, \\
        \sum_a \mu(s,a) = \sum_{s', a} \mu(s', a) p'(s',a,s), \quad \forall s \in \mathcal{S}, \label{eq3'} \\
        \mu(s,a) \geq 0, \quad \forall (s,a) \in \mathcal{S} \times \mathcal{A}, \quad \sum_{s,a} \mu(s,a) = 1, \label{eq4'}
    \end{align}
    Note that program (\ref{eq1'})-(\ref{eq4'}) is not linear anymore as $\mu(s', a)$ is being multiplied by $p'(s',a,s)$ in equation (\ref{eq3'}). This is a serious drawback of \textsc{UCRL-CMDP} algorithm because, as we show later, program (\ref{eq1'})-(\ref{eq4'}) becomes computationally inefficient for even moderate problems.
\end{enumerate}

\noindent In all three cases, we use the original bonus terms $b_t(s,a)$ and refer to the corresponding papers for more details regarding the definition of these terms.

\subsection{Environments}
\label{subsec:envs}

We consider three gridworld environments in our analysis. There are four actions possible in each state, $\mathcal{A} = \{up, down, right, left\}$, which cause the corresponding state transitions, except that actions that would take the agent to the wall leave the state unchanged. Due to the stochastic environment, transitions are stochastic (i.e., even if the agent's action is to go \textit{up}, the environment can send the agent with a small probability \textit{left}). Typically, the gridworld is an episodic task where the agent receives cost 1 (equivalently reward -1) on all transitions until the terminal state is reached. We reduce the episodic setting to the infinite-horizon setting by connecting terminal states to the initial state. Since there is no terminal state in the infinite-horizon setting, we call it the goal state instead. Thus, every time the agent reaches the goal, it receives a cost of 0, and every action from the goal state sends the agent to the initial state. We introduce constraints by considering the following specifications of a gridworld environment: Marsrover and Box environments.

\paragraph{Marsrover.} This environment was used in \cite{tessler2018reward, pmlr-v120-zheng20a, NEURIPS2020_Brantley}. The agent must move from the initial position to the goal avoiding risky states. Figure (\ref{marsrover_envs}) illustrates the CMDP structure: the initial position is light green, the goal is dark green, the walls are gray, and risky states are purple. "In the Mars exploration problem, those darker states are the states with a large slope that the agents want to avoid. The constraint we enforce is the upper bound of the per-step probability of step into those states with large slope -- i.e., the more risky or potentially unsafe states to explore" \citep{pmlr-v120-zheng20a}.  Each time the agent appears in a purple state incurs an auxiliary cost of 1. Other states incur no auxiliary costs.

\begin{figure}[t]
  \centering
  \subfigure[Marsrover 4x4]{
    \centering
    \includegraphics[width=4cm]{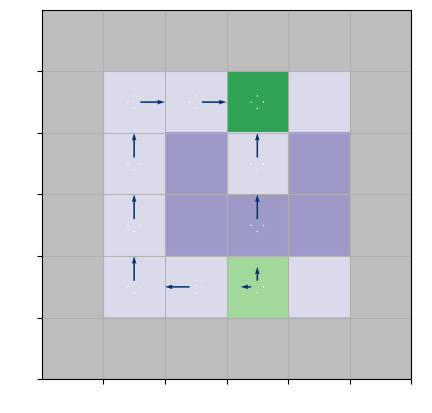}
    \label{4x4_marsrover}
    }
  \subfigure[Marsrover 8x8]{
    \centering
    \includegraphics[width=4cm]{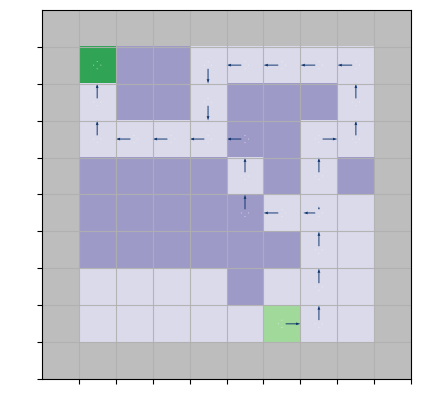}
    \label{8x8_marsrover}
  }
  \caption{Marsrover gridworlds. The initial position is light green, the goal is dark green, the walls are gray, and risky states are purple. Figure \ref{4x4_marsrover} illustrates 4x4 Marsrover environment. Figure \ref{8x8_marsrover} illustrates 8x8 Marsrover environment. In both cases, the agent's task is to get from the initial state to the goal state, and the optimal policy combines with some probabilities fast and safe ways, which are indicated by arrows on the pictures.}
\label{marsrover_envs}
\end{figure}

\begin{figure}[t]
  \centering
  \subfigure[Box (Main)]{
    \centering
    \includegraphics[width=3.8cm]{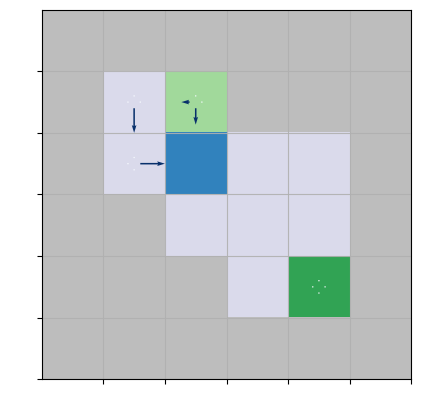}
    \label{box_main}
  }
  \subfigure[Box (Safe)]{
    \centering
    \includegraphics[width=3.8cm]{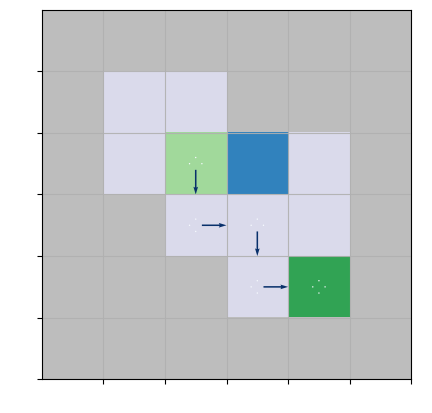}
    \label{box_left}
  }
  \subfigure[Box (Fast)]{
    \centering
    \includegraphics[width=3.8cm]{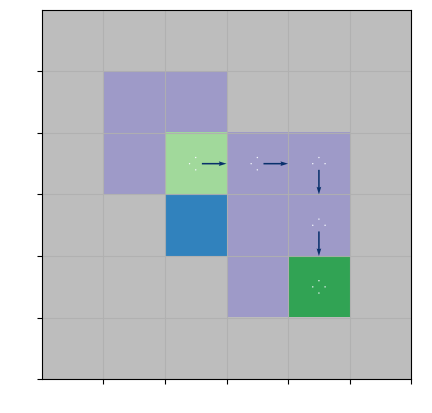}
    \label{box_down}
  }
  \caption{Box gridworld. The initial position is light green, the goal is dark green, the walls are gray, and risky states are purple. Figure \ref{box_main} illustrates the initial configuration. The agent's task is to get from the initial state to the goal state, and the optimal policy combines with some probabilities fast and safe ways, which are indicated by arrows on the pictures. Figure \ref{box_left}-\ref{box_down} illustrates safe and fast ways correspondingly.}
\label{box_envs}
\end{figure}

Without constraints, the optimal policy is to always go \textit{up} from the initial state. However, with constraints, the optimal policy is a randomized policy that goes \textit{left} and \textit{up} with some probabilities, as illustrated in Figure \ref{4x4_marsrover}. In experiments, we consider two marsrover gridworlds: 4x4, as shown in Figure \ref{4x4_marsrover}, and 8x8, depicted in Figure \ref{8x8_marsrover}.
\paragraph{Box.} Another conceptually different specification of a gridworld is Box environment from \cite{Leike_aisafegrid_2017}. Unlike the Marsrover example, there are no static risky states; instead, there is an obstacle, a box, which is only "pushable" (see Figure \ref{box_main}). Moving onto the blue tile (the box) pushes the box one tile into the same direction if that tile is empty; otherwise, the move fails as if the tile were a wall. The main idea of Box environment is "to minimize effects unrelated to their main objectives, especially those that are irreversible or difficult to reverse" \citep{Leike_aisafegrid_2017}. If the agent takes the fast way (i.e., goes down from its initial state; see Figure \ref{box_down}) and pushes the box into the corner, the agent will never be able to get it back, and the initial configuration would be irreversible. In contrast, if the agent chooses the safe way (i.e., approaches the box from the left side), it pushes the box to the reversible state (see Figure \ref{box_left}). This example illustrates situations of performing a task without breaking a vase, scratching the furniture, bumping into humans, etc.

Each action incurs an auxiliary cost of 1 if the box is in a corner (cells adjacent to at least two walls) and no auxiliary costs otherwise. Similarly to the Marsrover example, without safety constraints, the optimal policy is to take a fast way (go down from the initial state). However, with constraints, the optimal policy is a randomized policy that goes down and left from the initial state.

\subsection{Simulation results}

\begin{figure}
    \centering
    \includegraphics[width=0.85\textwidth]{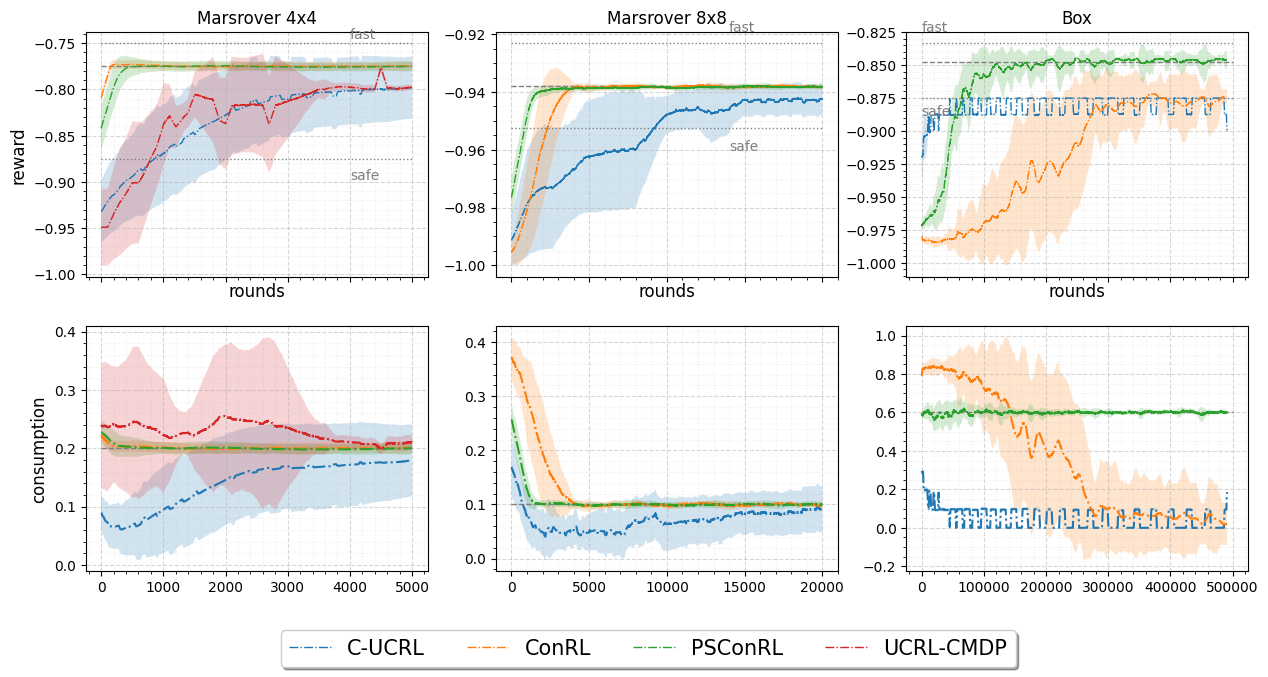}
    \caption{
    \textbf{(Top row)} shows the average reward (inverse average main cost); the dashed line shows the optimal behavior, and the dotted lines depict the reward level of safe and fast policies. \textbf{(Bottom row)} shows the average consumption of the auxiliary cost; the constraint thresholds are 0.2 for Marsrover 4x4, 0.1 for Marsrover 8x8, and 0.6 for Box.  Results are averaged over 100 runs for Marsrover 4x4 and over 30 runs for Marsrover 8x8 and Box.}
    \label{fig:sum_results_app}
\end{figure}

Figure \ref{fig:sum_results_app} shows the reward (inverse main cost) and average consumption (auxiliary cost) behavior of \textsc{PSConRL}, \textsc{ConRL}, \textsc{C-UCRL}, and \textsc{UCRL-CMDP} illustrating how the regret from Figure \ref{fig:sum_results} is accumulated. The top row shows the reward performance. The bottom row presents the average consumption of the auxiliary cost. 

Taking a closer look at Marsrover environments (left and middle columns), we see that all algorithms converge to the optimal solution (top row), and their average consumption (middle row) satisfies the constraints in the long run. In the Box example (right column), we see that \textsc{ConRL} and \textsc{C-UCRL} are stuck with the suboptimal solution. Both algorithms exploit safe policy once they have learned it, which corresponds to the near-linear regret behavior in Figure \ref{fig:sum_results}. Alternatively, \textsc{ConPSRL} converges to the optimal solution relatively quickly (middle and bottom graphs).
\end{document}